%% file: main.tex
\documentclass[conference]{IEEEtran}
\IEEEoverridecommandlockouts

\usepackage[utf8]{inputenc}
\usepackage[T1]{fontenc}
\usepackage{cite}
\usepackage{amsmath,amsfonts,mathtools}

\usepackage{amssymb}
\usepackage{amsthm}
\usepackage{graphicx}
\usepackage{textcomp}
\usepackage{xcolor}

\usepackage{booktabs}
\usepackage{longtable} %
\usepackage{multirow}
\usepackage{makecell}
\usepackage{diagbox}
\usepackage{rotating}
\usepackage{nicefrac}
\usepackage{textgreek}
\usepackage{comment}
\usepackage{listings}
\usepackage{tcolorbox}
\usepackage{mfirstuc}
\usepackage{xspace}
\usepackage[ruled,vlined,linesnumbered]{algorithm2e}
\usepackage{algorithmic}

\usepackage[disable]{todonotes}

\usepackage{caption}
\usepackage{subcaption}

\usepackage[hidelinks]{hyperref}
\usepackage[nameinlink,capitalize]{cleveref}

\def\BibTeX{{\rm B\kern-.05em{\sc i\kern-.025em b}\kern-.08em
    T\kern-.1667em\lower.7ex\hbox{E}\kern-.125emX}}

\theoremstyle{plain}
\newtheorem{theorem}{Theorem}[section]
\newtheorem{proposition}[theorem]{Proposition}

\theoremstyle{definition}

\newtheorem{assumption}[theorem]{Assumption}
\theoremstyle{remark}

\input{macro}

\crefformat{section}{§#2#1#3}

\providecommand{\citep}[1]{\cite{#1}}
\providecommand{\citet}[1]{\cite{#1}}

\title{\name{}: Protecting Dataset-Level Secrets in \\Textual Data Sharing}

\author{
\IEEEauthorblockN{Shuaiqi Wang\IEEEauthorrefmark{1},
Zinan Lin\IEEEauthorrefmark{2}, and
Giulia Fanti\IEEEauthorrefmark{1}}
\IEEEauthorblockA{\IEEEauthorrefmark{1}Carnegie Mellon University \quad
\IEEEauthorrefmark{2}Microsoft Research\\
\{shuaiqiw,gfanti\}@andrew.cmu.edu \quad zinanlin@microsoft.com}
}

\begin{document}

\maketitle

\input{sections/abstract}
\input{sections/intro}
\input{sections/related_work}

\input{sections/formulation}
\input{sections/method}

\input{sections/experiment}

\input{sections/conclusion}

\bibliographystyle{IEEEtran}
\input{references.bbl}

\appendices
\input{sections/app_candidate_distribution}
\input{sections/app_proof_sml}

\end{document}

%% file: macro.tex
\crefname{section}{\S}{\S}
\Crefname{section}{\S}{\S}
\crefname{appendix}{App.}{Apps.}
\Crefname{appendix}{App.}{Apps.}
\crefname{theorem}{Thm.}{Thms.}
\Crefname{theorem}{Thm.}{Thms.}
\crefname{proposition}{Prop.}{Props.}
\Crefname{proposition}{Prop.}{Props.}
\crefname{algorithm}{Alg.}{Algs.}
\Crefname{algorithm}{Alg.}{Algs.}
\crefname{assumption}{Asm.}{Asms.}
\Crefname{assumption}{Asm.}{Asms.}
\crefname{mechanism}{Mech.}{Mechs.}
\Crefname{mechanism}{Mech.}{Mechs.}

\newcounter{packednmbr}

\newcommand{\calD}{\mathcal{D}}

\newcommand{\calM}{\mathcal{M}}

\newcommand{\calX}{\mathcal{X}}
\newcommand{\calY}{\mathcal{Y}}
\newcommand{\calZ}{\mathcal{Z}}

\newcommand{\bra}[1]{\left( #1 \right)}
\newcommand{\brb}[1]{\left[ #1 \right]}

\newcommand{\brc}[1]{\left\{ #1 \right\}}
\newcommand{\brd}[1]{\left| #1 \right|}

\newcommand{\probnotation}{\mathbb{P}}
\newcommand{\probof}[1]{\probnotation\bra{#1}}

\newcommand{\secretrv}{g}
\newcommand{\secretvalueset}{\mathbf{G}}

\newcommand{\prop}{\phi}
\newcommand{\propsize}{\psi}

\newcommand{\content}{Y}
\newcommand{\style}{Z}
\newcommand{\contentset}{\mathcal{Y}}
\newcommand{\styleset}{\mathcal{Z}}
\newcommand{\influence}{L}

\newcommand{\fid}{FID}
\newcommand{\knnprecision}{KNN-Precision}
\newcommand{\knnrecall}{KNN-Recall}
\newcommand{\am}{AM}
\usepackage{bbm}

\newcommand{\name}{QuanText\xspace}

%% file: sections/abstract.tex
\begin{abstract}
Natural-language datasets support many downstream applications and research studies. However, released text can reveal sensitive \emph{global} properties of the data source, such as the proportion of records associated with a particular gender, diagnosis, or political stance. Existing work has largely focused on property inference attacks that recover such global properties, while defenses for protecting these dataset-level secrets remain limited. Differential privacy, though effective for protecting individual records, provides only weak protection for aggregate properties. In this work, we propose Randomized Quantization for Text (\name{}), a training-free and large-language-model-agnostic data release mechanism that protects global secrets in textual datasets while preserving data utility. Given {a function of the dataset that represents the secret in question (e.g. proportion of diabetics) and} attributes over which the data holder wants to retain utility {(e.g., topic and sentiment)}, \name{} distorts the distribution of the secret quantity \emph{and} the distributions of correlated attributes by (i) constructing a set of candidate release distributions {over attributes (both secret and non-secret)}, (ii) randomly selecting a distribution that is {sufficiently} close to the (private) empirical distribution, and (iii) rewriting each private text sample to match the chosen {attribute} distribution using attribute-related snippets from the original text. %
\name is inspired by the Statistic Maximal Leakage (SML) framework, which bounds leakage about a secret function of a data distribution. 
We show that under ideal conditions, \name theoretically satisfies an SML guarantee. 
Since these ideal conditions may not hold in practice, 
experiments on real-world datasets demonstrate that \name{} achieves a better empirical privacy–utility trade-off than competing data generation baselines. The code for \name{} is available at \url{https://github.com/wsqwsq/QuanText}.
\end{abstract}

\begin{IEEEkeywords}
Dataset-level privacy, data sharing, Statistic Maximal Leakage, Randomized Quantization.
\end{IEEEkeywords}

%% file: sections/intro.tex
\section{Introduction}

Natural-language datasets are important for both practical applications and academic research~\cite{bender2018data,gebru2021datasheets}. However, their release also raises significant privacy concerns. Beyond the leakage of sensitive information from individual samples~\cite{carlini2021extracting,lukas2023analyzing}, a released dataset may reveal sensitive global properties~\cite{hilprecht2019_property_gans,huang2025can,wang2024statistic}, such as the proportions of samples associated with certain attribute categories. For example, in medical records or patient--doctor dialogues~\cite{li2023chatdoctor}, even after individual identifiers have been removed, the released corpus may expose the fraction of female patients or the prevalence of specific diagnostic categories, revealing sensitive population-level demographic or health information. %
Similarly, in social-media datasets~\cite{StanceSemEval2016}, released text may reveal politically sensitive aggregate information, such as the proportion of posts supporting a controversial topic. Protecting such dataset-level properties can be important for textual data sharing, while the released data should still preserve realistic semantics to remain useful (\cref{fig:problem}). %

\begin{figure}[htbp]
    \centering
    \includegraphics[width=0.94\linewidth]{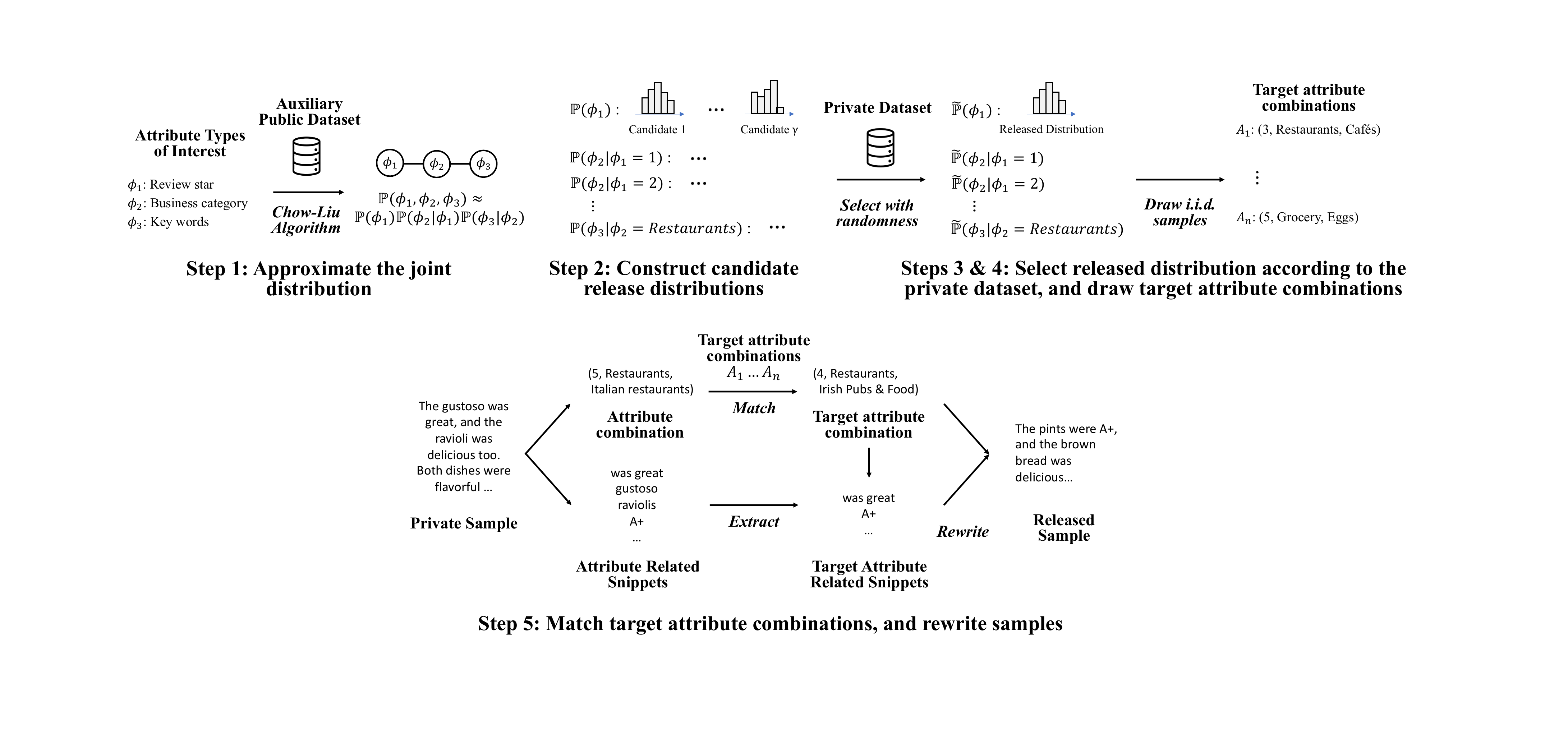}
    \caption{Textual data sharing can reveal sensitive global properties, such as the proportion of patients with a certain disease. We aim to design a data release mechanism that protects such global secrets while preserving the utility of released text. %
    }
    \label{fig:problem}
\end{figure}

Several studies have proposed property inference attacks that infer global properties of released data or of the generative models used in the release process~\cite{hilprecht2019_property_gans,prisampler2024,wei2024_property_existence_usenix,huang2025can}. However,  comparatively few works focus on defenses that protect such sensitive global properties~\cite{prisampler2024}. Prior studies have shown that differential privacy (DP), while effective for protecting individual samples, can be insufficient for protecting aggregate properties~\cite{ateniese2013_hacking_smart_machines,suri2023_dissecting,mothukuri2022_distribution_leakage_cross_silo,cretu2024_correlation_inference_sciadv}. 
Intuitively, DP perturbs individual samples independently, but because the added noise is typically zero-mean, its effects largely cancel out in aggregate, leaving the proportions of the sensitive attribute categories close to their original values. %
\cite{prisampler2024} proposes a simple defense that distorts the secret values in released data in the vision setting. However, this approach remains insufficient when other attributes in the dataset are correlated with the secret, since attackers can exploit these correlations to recover the sensitive global property. Designing defenses in textual settings is similarly challenging: first, a single attribute may be expressed across multiple parts of a text; second, textual attributes are often highly correlated, allowing attackers to recover the global secret from seemingly non-sensitive words, phrases, or semantic cues.

To this end, we propose Randomized Quantization for Text (\name{}), a training-free and large-language-model-agnostic data release mechanism that protects global {secret proportions of a dataset} while limiting leakage through correlated attributes and preserving the utility of released data. At a high level, \name{} proceeds in three stages. 
(i) First, \name{} efficiently constructs a set of candidate release distributions over the specified attribute types the data holder wants to retain utility on. For example, a user with a dataset of tweets may wish to preserve the distributions of topics and sentiments, both of which may be correlated with the target secret: the proportion of tweets favoring a particular political stance. Since directly releasing the private data (or even the distribution of these attributes) may leak the secret through their correlations, \name{} instead builds alternative candidate distributions with different secret and correlated-attribute marginal distributions. In our example, each candidate distribution 
{specifies a joint distribution over topics and sentiments.}
To construct these candidates, \name{} uses an auxiliary public dataset: it applies the Chow--Liu algorithm~\cite{chow1968approximating} to approximate the joint distribution {of attributes}
as a product {distribution} of simpler marginal and conditional distributions. 
(ii) Second, given the private dataset, \name{} determines the released joint distribution by randomly selecting from the subset of candidate distributions closest to the private empirical distribution. (iii) Third, \name{} constructs a released dataset that follows the selected attribute distribution. For each private sample, an attribute combination is assigned according to the release distribution, and a large language model rewrites the sample using the assigned attributes together with the corresponding attribute-relevant snippets extracted from the original text. For example, suppose a private tweet has the attributes $(\text{topic}=\text{climate change}, \text{sentiment}=\text{positive})$ and is assigned a synthetic attribute combination $(\text{topic}=\text{climate change}, \text{sentiment}=\text{negative})$. \name{} extracts relevant snippets from the original tweet, such as ``global warming'' and ``clean energy'', and uses them to generate a new tweet that expresses the assigned topic and sentiment.

Under idealized conditions, we prove that \name{} satisfies a privacy guarantee based on the Statistic Maximal Leakage (SML) framework~\cite{wang2024statistic}. 
SML bounds the leakage of a specified global-level secret under the worst-case data prior and against arbitrary attack strategies, including those that exploit correlated attributes. 
Since SML is robust to post-processing \cite{wang2024statistic}, a dataset released under a $\Pi-$SML guarantee can be used for downstream tasks like model training without degrading privacy.

Our theoretical analysis accounts for correlation between textual content and style, such as word choice {within individual text samples} or text length. We quantify content--style dependence by a worst-case conditional probability ratio $l$, and prove that the SML is bounded by the mechanism-dependent quantization term plus $\log l$, so weak dependence between content and style causes only limited additional privacy loss.

Since the correlation parameter is difficult to estimate in practice, 
we compare the empirical privacy and utility performance of \name
to several common data release baselines. 
Our results show that \name{} achieves a better empirical privacy–utility trade-off than competing data generation baselines on known property inference attacks.

Our contributions are summarized as follows.

\begin{itemize}
\item \textbf{Mechanism Design:} We propose \name{}, a model-agnostic, training-free data release mechanism that protects sensitive global {secret proportions of a dataset} while
limiting leakage through correlated attributes and preserving
data utility.  

\item \textbf{Privacy Analysis:} We analyze the idealized privacy guarantee of \name{} via Statistic Maximal Leakage (SML), which is robust to arbitrary attack strategies and further processing of the released data. We provide an SML guarantee whose additive degradation is controlled by the strength of content--style correlation.

\item \textbf{Empirical Evaluation:} We evaluate the privacy and utility performance of \name{} on real-world datasets. Compared with existing data generation baselines, \name{} achieves a better privacy–utility trade-off.
\end{itemize}

%% file: sections/related_work.tex
\section{Related Work}

\noindent\textbf{Attribute Inference Attacks.~}
Attribute inference attacks aim to recover a missing or sensitive attribute of an individual sample from its observed features, thus concern \emph{sample-level privacy}~\cite{fredrikson2015model,zhao2021feasibility,mehnaz2022your,jayaraman2022attribute,duddu2022inferring}. Existing studies typically assume access to a trained classifier. Our setting is fundamentally different. We study \emph{dataset} or \emph{distribution-level privacy}, where the sensitive information is a global property of the private dataset, such as the proportion of samples with a particular attribute.

\noindent\textbf{Property  Inference Attacks.~} 
Property inference attacks, also called distribution inference attacks, aim to infer aggregate properties of a private training dataset or distribution, such as demographic or class-label proportions, rather than attributes of individual records. Most prior work~\cite{snap_sp23_chaudhari2023,ateniese2013_hacking_smart_machines,ganju2018_property_inference,sone2022_formalizing_distribution_risks,zhang2021_leakage_multiparty,distribution_inference_risks_2022,suri2023_dissecting,mothukuri2022_distribution_leakage_cross_silo} targets discriminative classifiers, including fully connected neural networks~\cite{ganju2018_property_inference}, convolutional neural networks~\cite{sone2022_formalizing_distribution_risks}, and classifiers in federated learning~\cite{mothukuri2022_distribution_leakage_cross_silo}. This line of work also shows that differential privacy (DP), while designed to protect individual records, provides only weak protection for aggregate statistics~\cite{ateniese2013_hacking_smart_machines,suri2023_dissecting,mothukuri2022_distribution_leakage_cross_silo}. \cite{cretu2024_correlation_inference_sciadv} similarly finds DP may leave attribute correlations vulnerable in classifier-based settings.

A smaller body of work studies property inference for generative models or their synthetic outputs~\cite{hilprecht2019_property_gans,prisampler2024,wei2024_property_existence_usenix,huang2025can}. For Generative Adversarial Networks~\cite{hilprecht2019_property_gans} and diffusion models~\cite{prisampler2024}, attackers typically observe generated samples and estimate the target property empirically, assuming the generative distribution reflects the private training distribution. \cite{wei2024_property_existence_usenix} studies the narrower task of property existence inference, which tests whether a target property appears in the training data, i.e., whether its proportion is nonzero, rather than estimating general proportions. More recently, \cite{huang2025can} studies property inference attacks against large language models in both black-box settings, where attackers label generated samples to estimate the secret, and gray-box settings, where attackers use model weights and auxiliary data to train shadow models that map model features to the secret.
Our setting is closest to property inference from generated data; however, while prior work primarily studies attacks, our work focuses on defense.

\noindent\textbf{Defenses Against Property Inference Attacks.~}
Existing defenses against property inference attacks mainly target classifiers, using techniques such as property unlearning and adversarial training~\cite{lessons_learned_pia_defenses,pi_regression_2023,distribution_inference_risks_2022}. These heuristic methods are model-dependent and require modifying the training procedure. In contrast, we consider a model-agnostic, training-free defense at the data-release stage.
\cite{prisampler2024} perturbs the sensitive property in generated data, but changing the secret alone is insufficient: an attacker can still recover it from other attributes correlated with it. Our work addresses this limitation by protecting fine-grained dataset-level proportions while limiting leakage through correlated attributes.

%% file: sections/formulation.tex
\section{Problem Formulation}

The data holder has a private textual dataset $\mathcal{D}=\{x_1, \ldots, x_n\}$ of size $n$, where each sample $x_i\in \calX$ is a text string from universe $\calX$. 
We assume there exists a set of
$m$ attribute categories of interest, denoted by $\brc{\prop_i}_{i=1}^m$ (e.g., $\phi_1=$ sentiment, $\phi_2=$ stance), which are known to the data holder and end users. 
Each attribute type $\prop_i$, such as sentiment, takes values from a predefined set of categories with size ${\propsize_i}$, %
denoted by $\brc{v_j^i}_{j=1}^{{\propsize_i}}$, such as positive, negative, and neutral. 
The data holder aims to release a dataset $\mathcal{D'}$ via a data generation mechanism {$\calD' = \calM(\calD)$} while protecting a {count query $G(\calD)$ measuring the proportion of samples associated with a particular (combination of) attribute category values. }
That is, suppose there exists an oracle $\mathfrak{g}:\calX \to \{0,1\}$, that outputs 1 if and only if the sample $x$ has a certain (combination of) attribute properties (e.g., ``male with diabetes''), and 0 otherwise. Then
$$
G(\calD) \triangleq \frac{1}{n}\sum_{x \in \calD} \mathfrak{g}(x).
$$
For brevity, we will use the shorthand $G$ to denote $G(\calD)$. 
{A core assumption of this work is that the secret $G$ can be %
recovered by knowing the values of all attributes $\{\phi_i\}_{i=1}^m$; 
even if the oracle $\mathfrak g$ depends only on a strict subset of the $m$ attribute categories, the secret $G$ may still be correlated with other attributes that are not evaluated by $\mathfrak g$.
We model leakage about the secret from non-attribute content in \cref{sec:theory}.}
Following \cite{huang2025can}, the dataset release mechanism may use a public auxiliary dataset $\mathcal{D}_{\text{aux}}$ from a similar domain.

After observing the released dataset, the attacker aims to infer the original secret $G$. %
We assume that the attacker knows %
the data generation mechanism and is computationally unbounded.

\paragraph{Privacy Constraint}
Theoretically, we measure the privacy of a data release mechanism $\mathcal M$ in the Statistic Maximal Leakage (SML) framework \cite{wang2024statistic}, which provides privacy guarantees under any data prior 
and against arbitrary attack strategies. Let $\mathcal{P}$ denote the prior distribution of data, %
$\mathcal{A}$ be the attack method, $G$ and $\hat{G}$ be the random variables representing the original and attacker-guessed secret values, and $\secretvalueset$ be the set of all possible secret values. SML is defined as
\begin{align}
\Pi_{\mathcal{M}, \mathfrak{g}} = \sup_{\mathcal{P}, \mathcal{A}}\log \frac{\probof{\hat{G}=G}}{\sup_{\secretrv\in \secretvalueset} \mathbb{P}_G\bra{\secretrv}}.
\label{eq:privmetric}
\end{align}

SML takes the worst-case leakage over all possible data priors and attack methods. Intuitively, SML measures the gain in the attacker's probability of correctly guessing the secret given the released dataset $\calD'$. A smaller value of $\Pi_{\mathcal{M}, \mathfrak{g}}$ indicates stronger protection; in particular, if $\Pi_{\mathcal{M}, \mathfrak{g}}\le \eta$, any attack success probability after release is at most $\alpha e^\eta$, where $\alpha$ is the success probability using prior knowledge alone. SML also satisfies post-processing and adaptive composition \cite{wang2024statistic}, making the privacy guarantee robust to further processing of the released output and sequential applications of data release mechanisms.

%% file: sections/method.tex
\section{Randomized Quantization for Text}

We design Randomized Quantization for Text (\name{}) as a data release mechanism that limits leakage about the 
secret quantity $G$
while preserving the utility of the attribute types of interest, $\brc{\prop_i}_{i=1}^m$. 

\paragraph{Straw Man Solution}
A straw-man design that simply removes samples containing the sensitive property (e.g., ``has diabetes'') is insufficient, because an attacker may still infer the secret from the distributions of correlated attributes (e.g., age), which may remain largely unchanged. To mitigate this risk, \name{} perturbs both the distribution of the sensitive property \emph{and} the distributions of attributes correlated with it.

\paragraph{Overview}
{Our core assumption is that the secret $G$ is revealed exactly by knowing the attribute values of all samples in the private dataset. 
Hence, our approach is to \emph{rewrite} samples in the dataset so that their attribute values, in aggregate, do not reveal the secret $G$. 
\name  extracts attributes values from each text sample in $\calD$, obtaining an \emph{empirical joint distribution} $\hat {\mathbb{P}}$ over these attribute values.
It also constructs a set of alternative joint distributions over attribute values. 
To do so, it models the underlying joint attribute distribution with a tree-structured probabilistic graphical model using the Chow-Liu algorithm \cite{chow1968approximating}.
Given a set of candidate distributions (or a \emph{quantization} of the attribute distribution space), \name selects uniformly at random from the set of candidate distributions that are closest in total variation distance to the empirical attribute distribution; the size of this subset can be varied to tune privacy guarantees. 
Once we draw an alternative attribute distribution $\tilde{\mathbb{P}}$, we use an LLM to rewrite each sample in $\calD$, mapping its attribute values to a synthetic combination of attribute values drawn from $\tilde{\mathbb{P}}$. 
To minimize the changes to the dataset, we use the Hungarian algorithm to find a low-cost matching between  samples' true attributes to a list of synthetic attributes; here, cost is defined as the number of differing attribute values.
}

\subsection{\name Algorithm}
We illustrate the pipeline of \name{} in \cref{fig:illu} and present the full algorithm in \cref{alg:main}.

\begin{figure*}[htbp]
    \centering
    \includegraphics[width=\linewidth]{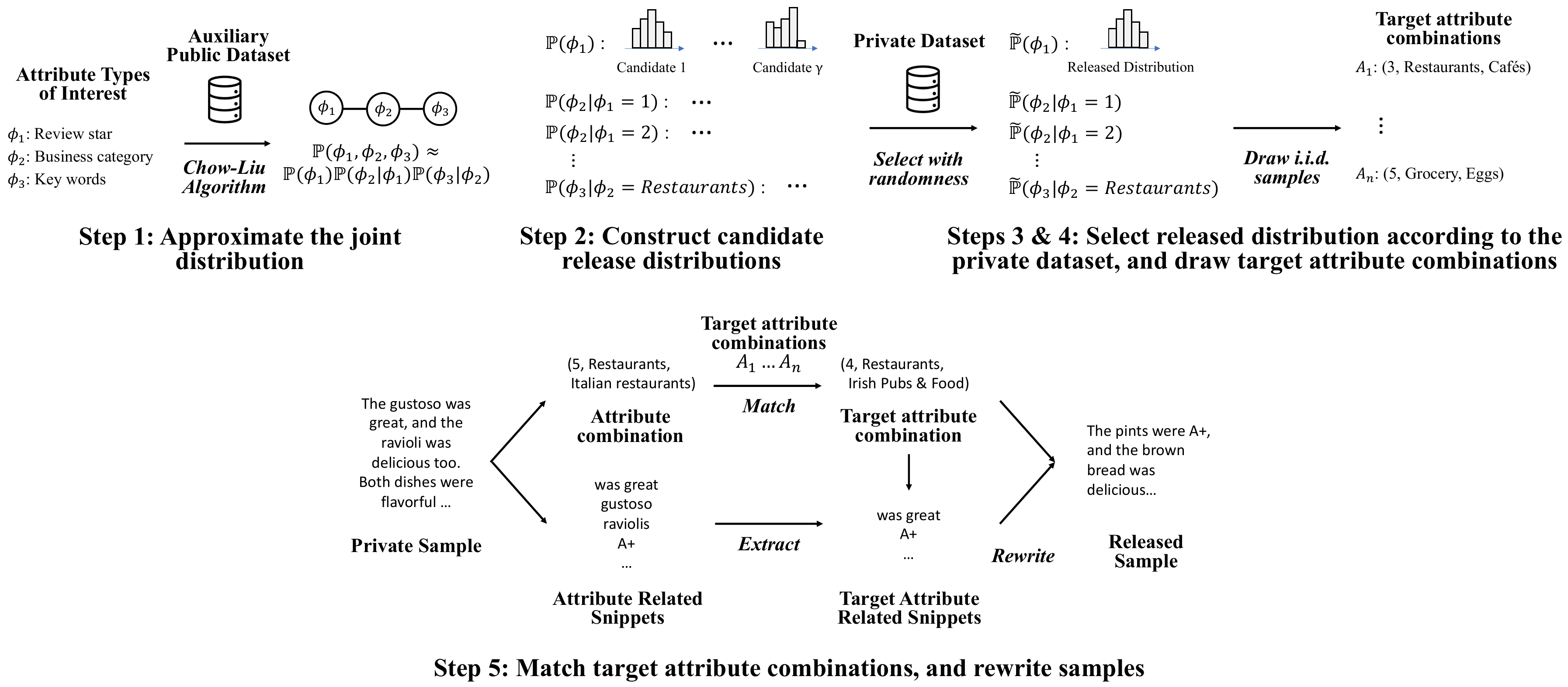}
    \caption{Illustration of \name{} on a customer review dataset, where the attribute types of interest are review star, business category, and key words. Step 1 approximates the joint attribute distribution with a Chow--Liu tree learned from an auxiliary public dataset. Step 2 constructs candidate release distributions for each distribution component. Steps 3 and 4 randomly select released distributions based on the private dataset and then sample target attribute combinations. Step 5 matches these target combinations to private samples, extracts snippets relevant to the assigned target attributes, and rewrites the samples accordingly to form the released dataset.}
    \label{fig:illu}
    \vspace{-3mm}
\end{figure*}

\begin{algorithm}[htbp]
\small
\DontPrintSemicolon
\SetAlgoLined
\SetKwInput{KwIn}{Input}
\SetKwInput{KwOut}{Output}

\KwIn{
Private dataset $\mathcal{D}=\brc{x_1,\ldots,x_n}$, auxiliary dataset $\mathcal{D}_{\mathrm{aux}}$,
attribute types of interest $\brc{\prop_1,\ldots,\prop_m}$,
number of candidate distributions $\gamma$,
selection size $k$, large language model $f$.
}
\KwOut{Released dataset $\mathcal{D}'$.}

\BlankLine
\tcp{Step 1: Chow--Liu approximation}
$\probof{\prop_1,\ldots,\prop_m}\approx \probof{\prop_1}\cdot\prod_{i=2}^{m}\probof{\prop_i|\prop_{\pi\bra{i}}}\leftarrow$ \textsc{ChowLiu} $(\mathcal{D}_{\mathrm{aux}})$.\; \label{alg:chowliu}

\BlankLine
\tcp{Step 2--4: Decide released distribution}
\For{each distribution component $\mathbb{P}\in
\brc{\mathbb{P}(\prop_1)}
\cup
\brc{\mathbb{P}(\prop_i\mid \prop_{\pi\bra{i}}=v^{\pi\bra{i}}_j):
i=2,\ldots,m,\ j\in\brb{\propsize_{\pi\bra{i}}}}$}{
$\mathcal{Q}\gets\textsc{CandidateConstruction}(\gamma,\mathbb{P})$ [\cref{alg:candidate_distribution}].\; \label{alg:step2}

Compute the empirical distribution $\hat{\mathbb{P}}$ from $\mathcal{D}$.\; \label{alg:step3}

$\mathcal{Q}_k \leftarrow \arg\min{\substack{\mathcal{Q}'\subseteq\mathcal{Q}: |\mathcal{Q}'|=k}} \sum_{q\in\mathcal{Q}'} d_{\mathrm{TV}}(q,\hat{\mathbb{P}})$.\; \label{alg:step4-1a}

$\tilde{\mathbb{P}}\gets\textsc{Uniform}(\mathcal{Q}_k)$.\; \label{alg:step4-1b}

}

\BlankLine
$\tilde{\mathbb{P}}(\prop_1,\ldots,\prop_m)
\gets
\tilde{\mathbb{P}}(\prop_1)\cdot
\prod_{i=2}^{m}\tilde{\mathbb{P}}(\prop_i\mid\prop_{\pi\bra{i}})$.\; \label{alg:step4-2}

Draw target attribute combinations
$\brc{\bra{\tilde{a}^1_j,\ldots,\tilde{a}^m_j}}_{j=1}^{n}
\overset{\mathrm{i.i.d.}}{\sim}
\tilde{\mathbb{P}}(\prop_1,\ldots,\prop_m)$.\; \label{alg:step4-3}

\BlankLine
\tcp{Step 5: Match and rewrite samples}
\For{$i\gets 1$ \KwTo $n$}{
$\bra{\hat{a}^1_i,\ldots,\hat{a}^m_i}\gets\textsc{LabelAttributes}_f(x_i)$.\; \label{alg:step5-1}
}

Set $C_{ij}\gets\sum_{r=1}^{m}\mathbbm{1}\left[\hat{a}^r_i\ne\tilde{a}^{r}_j\right]$ 
for all $i,j\in\brc{1,\ldots,n}$.\; \label{alg:step5-12}

$\sigma\gets\textsc{Hungarian}(C)$, where $\sigma(i)$ is the target combination assigned to $x_i$.\; \label{alg:step5-13}

\For{$i\gets 1$ \KwTo $n$}{
$s_i\gets\textsc{ExtractSnippets}_f(x_i,\bra{\tilde{a}^1_{\sigma(i)},\ldots,\tilde{a}^m_{\sigma(i)}})$.\; \label{alg:step5-2}

$x'_i\gets\textsc{Rewrite}_f(x_i,s_i,\bra{\tilde{a}^1_{\sigma(i)},\ldots,\tilde{a}^m_{\sigma(i)}})$.\; \label{alg:step5-3}

}

\BlankLine
$\mathcal{D}'\gets\brc{x'_1,\ldots,x'_n}$.\;

\Return $\mathcal{D}'$.\;

\caption{Randomized Quantization for Text}
\label{alg:main}
\end{algorithm}
\normalsize

\begin{enumerate}
\item Approximate the joint distribution over the attribute types of interest,
$\probof{\prop_1, \prop_2, \ldots, \prop_m}$, as
$\probof{\prop_1}
\cdot\prod_{i=2}^m \probof{\prop_i|\prop_{\pi\bra{i}}}$
using the Chow--Liu algorithm~\cite{chow1968approximating} on an auxiliary public dataset $\mathcal{D}_{\text{aux}}$ (\cref{alg:chowliu}).
Here, $\prop_{\pi\bra{i}}$ denotes the parent of $\prop_i$ in the Chow--Liu tree, and $\prop_1$ is the root node. \label{step1}
    
\item Construct $\gamma$ candidate release distributions per distribution component, %
namely $\probof{\prop_1}$ and $\probof{\prop_i|\prop_{\pi\bra{i}}}$ for all $i\in\brc{2,3,\ldots,m}$ (\cref{alg:step2}). We defer the construction algorithm to \cref{app:candidate_distribution}. \label{step2}

\item Compute the empirical distributions from the original private dataset: $\hat{\probnotation}\bra{\prop_1}$ and $\hat{\probnotation}\bra{\prop_i|\prop_{\pi\bra{i}}}$ for all $i\in\brc{2,3,\ldots,m}$ (\cref{alg:step3}). %
\label{step3}

\item Select release distributions for the distribution components, combine them into a released joint distribution, and construct target attribute combinations: \label{step4} 
\begin{itemize}  
\item For each distribution component, determine the released distribution
$\tilde{\probnotation}\bra{\prop_1}$ or
$\tilde{\probnotation}\bra{\prop_i|\prop_{\pi\bra{i}}}$,
$i\in\brc{2,3,\ldots,m}$ by (a) identifying the top-$k$ candidate distributions with the smallest total variation (TV) distances from the corresponding empirical private distribution (\cref{alg:step4-1a}), and (b) selecting one of them uniformly at random (\cref{alg:step4-1b}).

    \item Form the released joint distribution (\cref{alg:step4-2}): $\tilde{\probnotation}\bra{\prop_1, \prop_2, \ldots, \prop_m} = \tilde{\probnotation}\bra{\prop_1} \cdot \prod_{i=2}^m \tilde{\probnotation}\bra{\prop_i|\prop_{\pi\bra{i}}}
    $.
    \item Draw $n$ target attribute combinations from the released joint distribution, where $n$ is the number of private samples (\cref{alg:step4-3}). Each combination specifies one category value per attribute type of interest.
\end{itemize}

\item Rewrite samples to match the released distribution using attribute-related snippets
from the original text: \label{step5}
\begin{itemize} 
\item Label the attribute combination of each private sample via an LLM (\cref{alg:step5-1}), and map the private samples to the target attribute combinations using the Hungarian algorithm (\cref{alg:step5-12,alg:step5-13}). 
\item For each sample, extract snippets relevant to its mapped target attribute combination using an LLM (\cref{alg:step5-2}). 
\item Rewrite each sample with an LLM according to the mapped target attribute combination and the extracted snippets (\cref{alg:step5-3}). 
\end{itemize}
\end{enumerate}

\subsection{Idealized Privacy Guarantee}
\label{sec:theory}
To analyze the Statistic Maximal Leakage (SML) of \name{}, following prior work~\cite{shen2017style,fu2018style,john2019disentangled}, we decompose textual data along two axes: content $Y\in\mathcal{Y}$  and style $Z\in\mathcal{Z}$ with $\calY \cup \calZ = \calX$. Content captures the semantic meaning of the text, including attribute-level information, whereas style refers to non-semantic surface properties, such as word choice, phrasing, and text length. Global-level secrets are functions of the private dataset and may be correlated with content; through content--style dependence, style may therefore also reveal information about the secret.

We quantify the dependence of style on content using the worst-case conditional probability ratio, denoted by $\influence\bra{\content; \style}$.
\begin{align}
    \influence\bra{\content; \style} \triangleq \sup_{y_1, y_2 \in \contentset; z \in \styleset} \frac{\probof{\style = z | \content = y_1}}{\probof{\style = z | \content = y_2}}.
\end{align}
Intuitively, $\influence\bra{\content; \style}$ measures how much the content $Y$ can change the likelihood of a particular style $Z$. When $\influence\bra{\content; \style}=1$, the style is independent of the content.
We show that $\influence\bra{\content; \style}$ upper bounds the mutual information $I\bra{\content; \style}$.

\begin{proposition}
If $\influence\bra{\content; \style} \leq l$, then $I\bra{\content; \style} \leq \log l.$ 
\end{proposition}

We assume that LLM used in \name{} can generate samples that satisfy the requirements specified in the prompts.

\begin{assumption}
[LLM Capability] 
\label{assumption:llm}
LLM rewrite faithfully realizes the assigned attribute combination: for every private
sample $x$, LLM $f$, extracted snippets $s$, and target combination
$a$, the rewritten sample
$x'=\textsc{Rewrite}_f(x,s,a)$ satisfies
$\mathfrak b(x', a)=1$, where the oracle $\mathfrak b$ outputs 1 if and only if a sample has property combination $a$. 
\end{assumption}

The following theorem characterizes the SML of \name{}.

\begin{theorem}[SML of \name{}]
For any secret defined as the proportion of samples associated with particular attribute categories in the private dataset $\mathcal{D}$, if $\influence\bra{\content; \style} \leq l$, then under \cref{assumption:llm}, the SML of \name{} satisfies
\vspace{-1mm}
\begin{align}
\label{eqn:sml}
    \Pi_{\mathcal{M}, \mathfrak{g}} \leq \log \bra{\frac{\gamma}{k}}^{1+\sum_{i=2}^m {\propsize_{\pi\bra{i}}}} + \log l,
\end{align}
where $\prop_{\pi\bra{i}}$ denotes the parent of $\prop_i$ in the constructed Chow--Liu tree, $\propsize_{\pi\bra{i}}$ is the number of values attribute $\prop_{\pi\bra{i}}$ can take, %
$\gamma$ is the number of candidate distributions for each distribution component, %
and $k$ is the selection size.
\label{thm:sml}
\end{theorem}

(Proof in \cref{app:proof_sml}.) \cref{thm:sml} shows that \name{} provides a stronger privacy guarantee, i.e., a smaller SML value, when {the number of candidate distributions}  $\gamma$ {per conditional attribute} is smaller, {the number of nearest-neighbor distributions} $k$ is larger, or the parent attributes in the Chow--Liu tree have fewer categories. This is consistent with the intuition that, when the mechanism selects from a larger fraction of the candidate distributions and each relevant attribute has fewer possible categories, the released distribution reveals less information for the attacker to use when inferring the secret. Furthermore, content--style correlation increases the privacy loss bound additively by $\log l$.

Since SML admits an operational interpretation, \cref{thm:sml} directly bounds the attacker’s success probability.

\begin{proposition}[Best attack success rate]
\label{prop:asr}
Let $\alpha$ denote the optimal attack success rate using prior knowledge alone. After observing the dataset released by \name{}, the optimal attack success rate satisfies
$$
\mathbb{P}(\hat{G}=G) \leq \alpha \cdot \bra{\frac{\gamma}{k}}^{1+\sum_{i=2}^m {\propsize_{\pi\bra{i}}}}.
$$
\end{proposition}

%% file: sections/experiment.tex
\section{Experiments}

We evaluate the privacy and utility of \name{} on real-world datasets.

\subsection{Datasets} 
We use the Tweet Stance \cite{StanceSemEval2016} and ChatDoctor \cite{li2023chatdoctor} datasets, and treat them as the real private data to be released and protected. %

\begin{itemize}
    \item \textbf{Tweet Stance Dataset \cite{StanceSemEval2016}~} This dataset contains 4,870 tweets annotated with three labels: \textit{Target} (five political topics: climate change, atheism, legalization of abortion, Hillary Clinton, and Donald Trump), \textit{Stance} (Favor, Against, Neither), and \textit{Sentiment} (Positive, Negative, Neither). We subsample 3,000 tweets as the private dataset and use the remaining 1,870 samples as the auxiliary dataset for the Chow--Liu spanning tree. We consider \textit{Target}, \textit{Stance}, and \textit{Sentiment} as the attributes of interest, yielding the Chow--Liu spanning tree \{\textit{Target} $\rightarrow$ \textit{Stance}; \textit{Target} $\rightarrow$ \textit{Sentiment}\}.
    We define the secrets as the proportions of tweets in a specific \textit{Target}, \textit{Stance}, or \textit{Sentiment} category, and in a specific \textit{Target}–\textit{Stance} combination. All secrets are specified with a precision of $0.01\%$.

    \item \textbf{ChatDoctor Dataset \cite{li2023chatdoctor}~} This dataset consists of patient--doctor dialogues. Following \cite{huang2025can}, we define two types of secrets: (1) the proportion of female samples and (2) the proportions of specific medical diagnoses. For (1), we subsample three datasets of 600 dialogues each with target female ratios 0.3, 0.5, and 0.7, taking gender as the only attribute of interest. For (2), we subsample 600 dialogues, define the secrets as the proportions of samples with mental disorder, digestive disorder, and childbirth, each to a precision of $0.01\%$, and take diagnosis as the only attribute of interest. In both cases, the single attribute of interest reduces the Chow--Liu tree to a single node.

\end{itemize}

\begin{figure*}[htbp]
    \centering
\begin{subfigure}{0.48\textwidth}
         \centering    \includegraphics[width=0.81\linewidth]{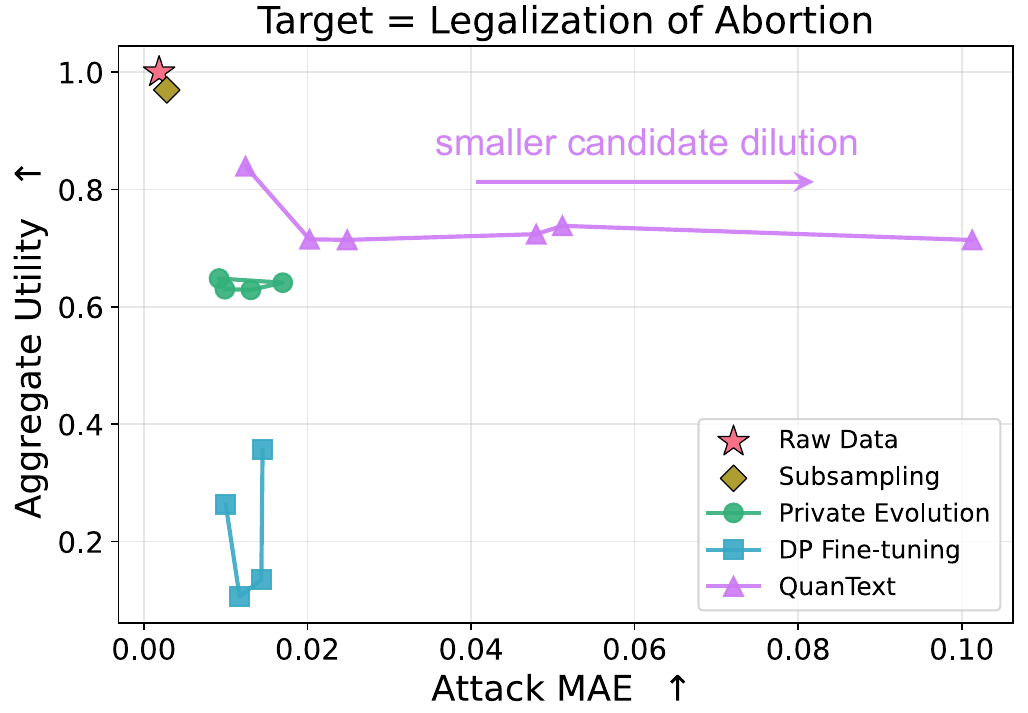}
    \caption{Secret as the proportion of Target = Legalization of Abortion.}
    \label{fig:pareto_target}
\end{subfigure}
\begin{subfigure}{0.48\textwidth}
         \centering
    \includegraphics[width=0.81\linewidth]{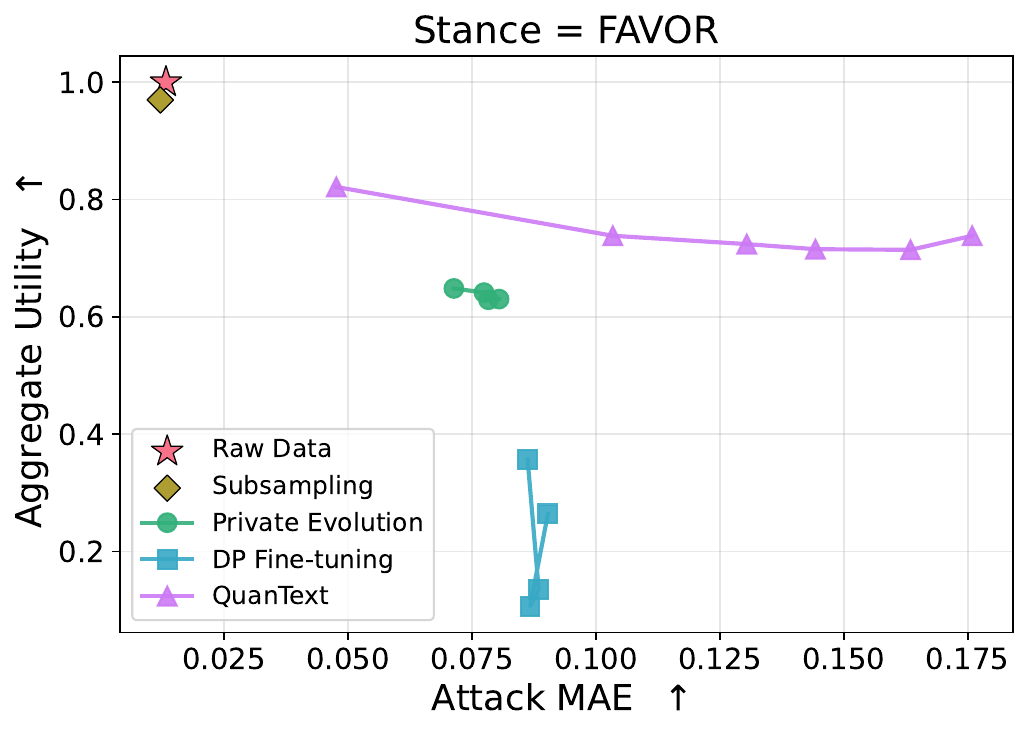}
    \caption{Secret as the proportion of Stance = Favor}
    \label{fig:pareto_stance}
\end{subfigure}
\begin{subfigure}{0.48\textwidth}
         \centering
    \includegraphics[width=0.81\linewidth]{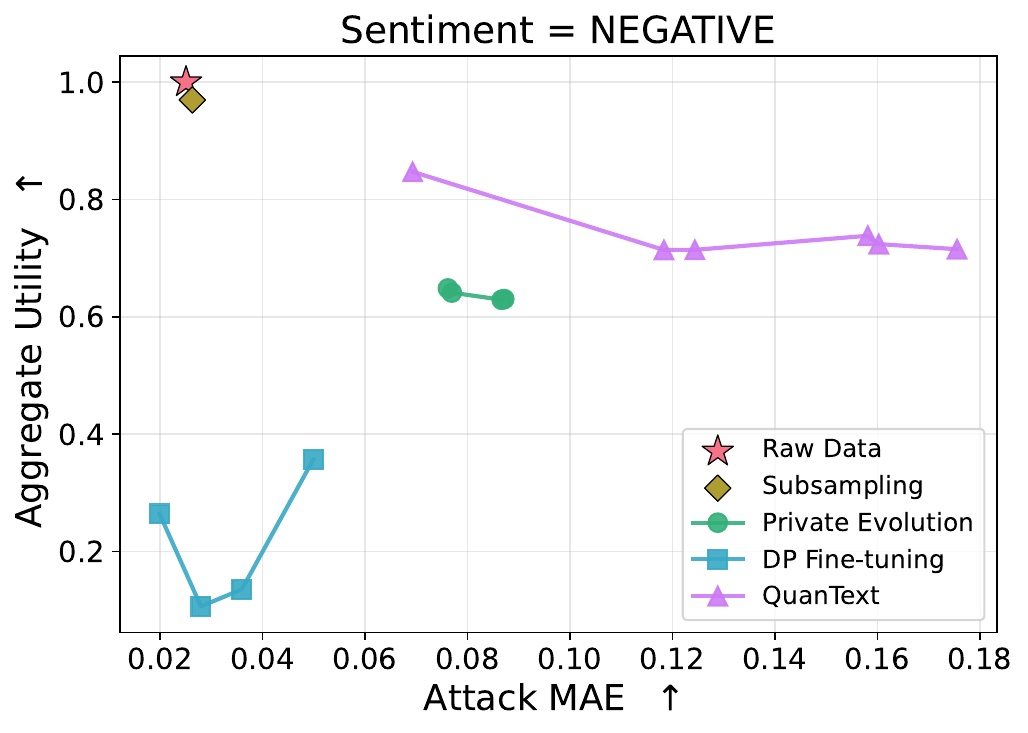}
    \caption{Secret as the proportion of Sentiment = Negative.}
    \label{fig:pareto_sentiment}
\end{subfigure}
\begin{subfigure}{0.48\textwidth}
         \centering
    \includegraphics[width=0.81\linewidth]{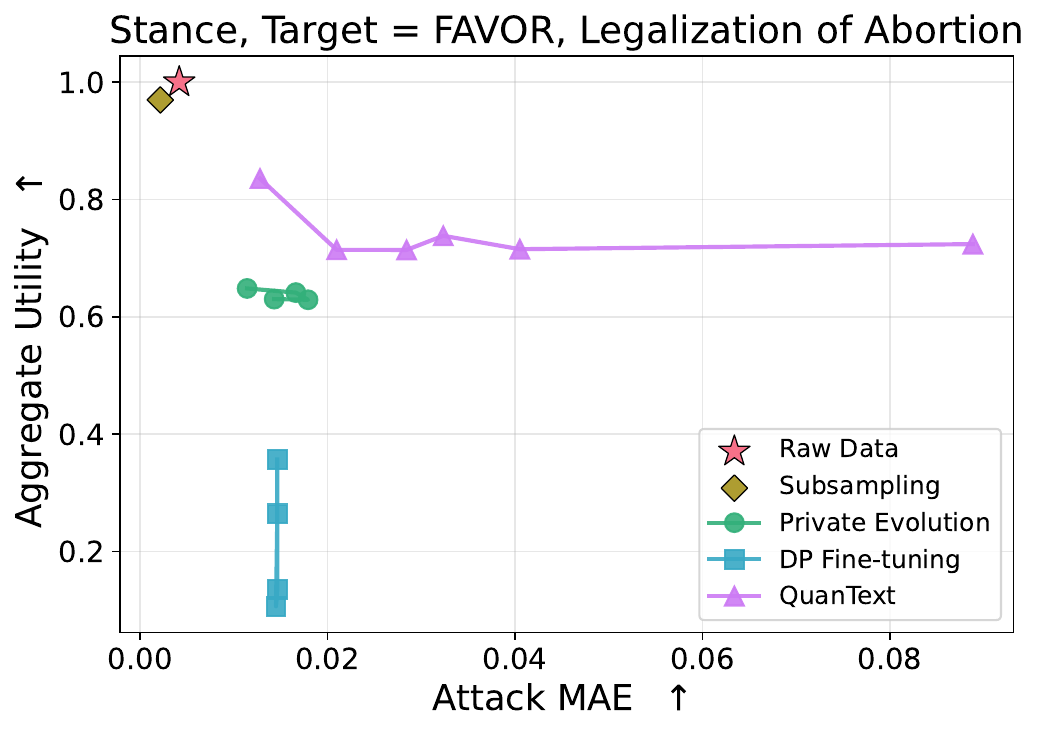}
    \caption{Secret as the proportion of Stance = Favor \& Target = Legalization of Abortion.
    }
    \label{fig:pareto_combined}
\end{subfigure}
\caption{Attack MAE vs. aggregate utility under different data generation methods with varying privacy guarantees on Tweet. 
}
\label{fig:pareto_tweet}
\vspace{-2mm}
\end{figure*}

\subsection{\name{} and Data Generation Baselines}  

As shown in \cref{thm:sml}, for a given dataset,  the SML of \name{} is determined by $\frac{\gamma}{k}$ under \cref{assumption:llm}. Intuitively, $\frac{\gamma}{k}$ measures the {degree of quantization among candidate distributions, coupled with how closely we stay in the neighborhood of the true empirical attribute distribution.} 
We refer to this ratio as \textbf{candidate dilution}. Although \cref{assumption:llm} may not hold exactly in practice, we vary the candidate dilution, i.e., $\frac{\gamma}{k}$, to control the privacy performance of \name{} in our experiments, where smaller candidate dilution corresponds to stronger privacy. Specifically, we consider $\frac{\gamma}{k} \in \brc{2,3,4,6,8,12}$ and use \texttt{Llama3.1-8B-Instruct} as the backend model. %

As {we do not know of any other} defenses specifically {against} target property inference attacks in textual {data release} settings, we select Differentially Private (DP) synthetic data generation methods as our baselines, specifically, Private Evolution \cite{lin2023differentially,xie2024differentially,lin2025differentially} and DP model fine-tuning \cite{yu2021differentially,wutschitz2022dp,yue2022synthetic}. We also include raw data release and subsampling as  naive baselines.

\begin{itemize}
\item \textbf{Raw Data Release:} We release the private data directly.

\item \textbf{Subsampling:} We randomly subsample and release half of the private dataset.

\item \textbf{Private Evolution (PE) \cite{lin2023differentially,xie2024differentially,lin2025differentially,tran2026differentially}:} PE is a training-free method for differentially private synthetic data generation using foundation models \cite{lin2023differentially,xie2024differentially,lin2025differentially,tran2026differentially,hou2024pre,gong2025dpimagebench}. Starting from samples generated by a Random API, PE iteratively constructs a DP noisy histogram from private-to-synthetic nearest-neighbor votes, samples from this histogram, and perturbs selected samples via a Variation API. We use Augmented Private Evolution (Aug-PE) \cite{xie2024differentially}, the text-generation variant of PE, modifying only the Random and Variation prompts.

\item \textbf{DP Fine-Tuning (DP-FT) \cite{yu2021differentially,wutschitz2022dp,yue2022synthetic}:} DP-FT fine-tunes the language model for next-token prediction using differentially private stochastic gradient descent (DP-SGD) \cite{abadi2016deep}. Synthetic data are then generated from the fine-tuned model according to a generation instruction.
\end{itemize}

We run PE for 10 iterations and DP-FT for 15 epochs. For both PE and DP-FT, we vary the privacy budget $\epsilon \in {1,2,3,4}$, %
and use \texttt{Llama3.1-8B-Instruct} as the backend model.

\begin{figure*}[htbp]
    \centering
\begin{subfigure}{0.48\textwidth}
         \centering
    \includegraphics[width=0.81\linewidth]{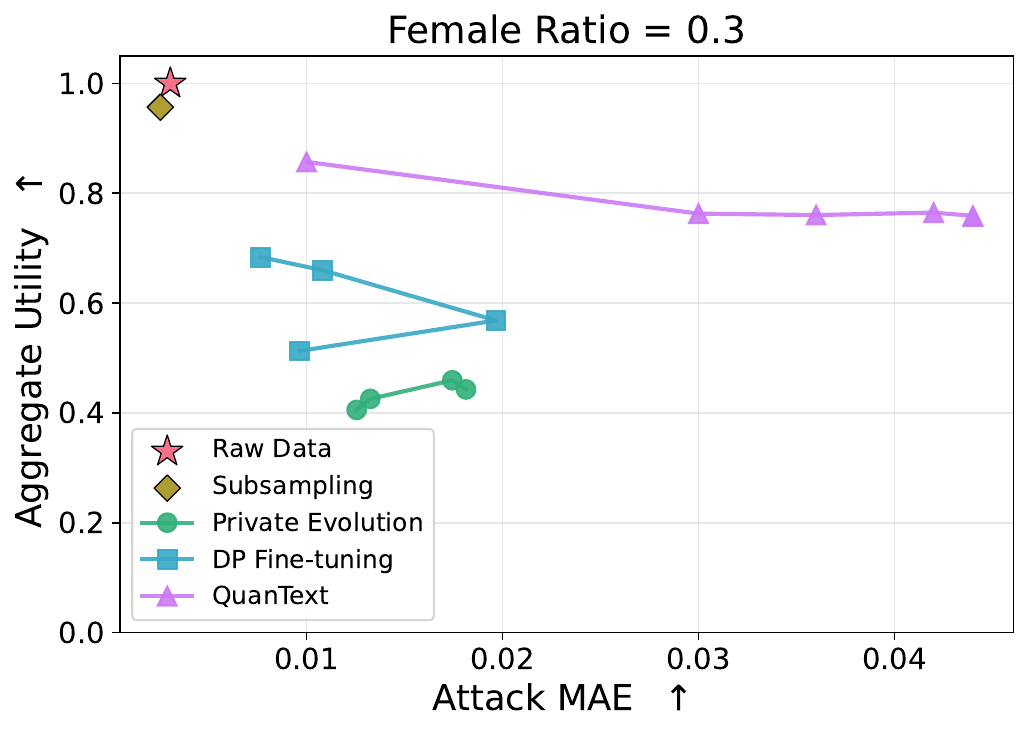}
    \caption{Secret as Female Ratio = 0.3.}
    \label{fig:pareto_gender3}
\end{subfigure}
\begin{subfigure}{0.48\textwidth}
         \centering
    \includegraphics[width=0.81\linewidth]{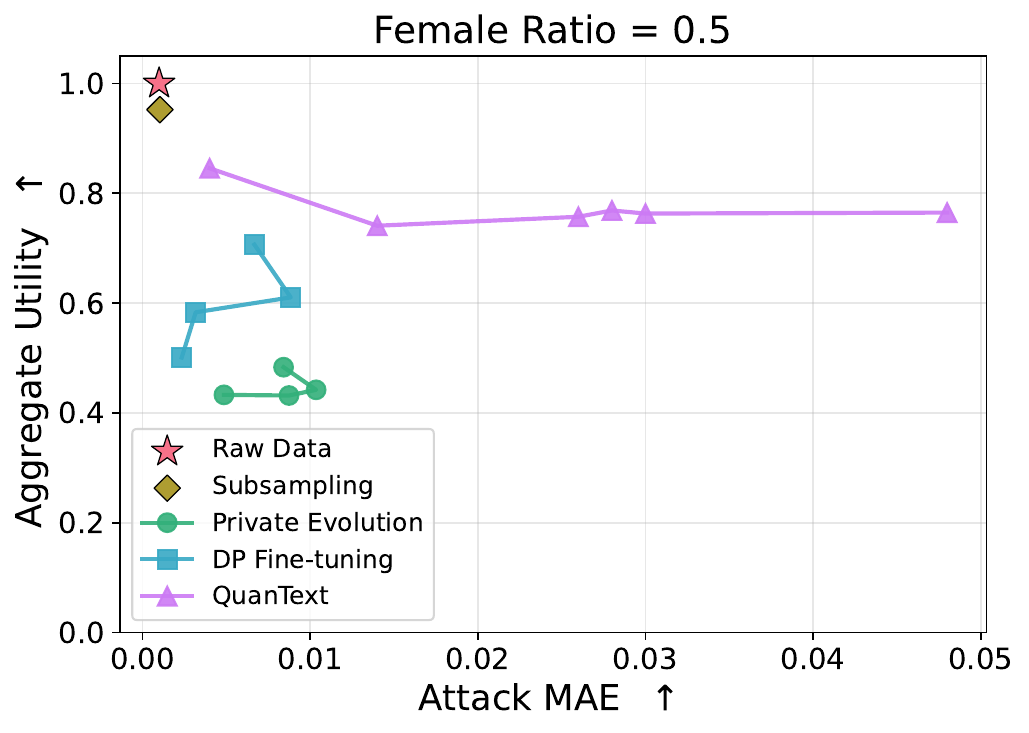}
    \caption{Secret as Female Ratio = 0.5.}
    \label{fig:pareto_gender5}
\end{subfigure}
\begin{subfigure}{0.48\textwidth}
         \centering
    \includegraphics[width=0.81\linewidth]{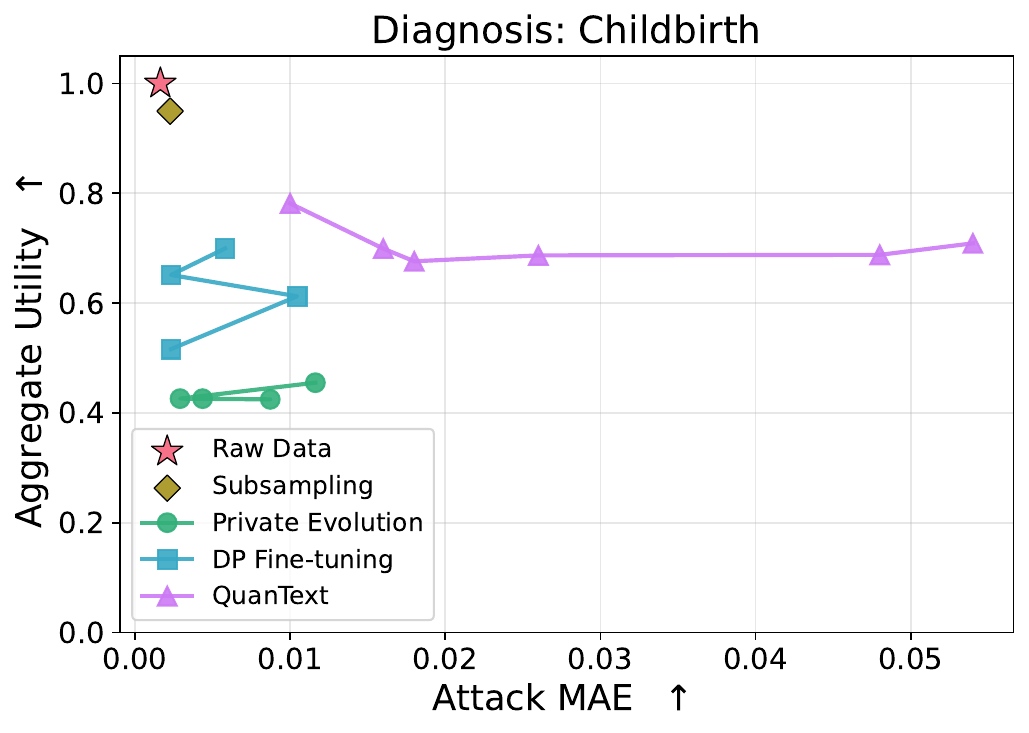}
    \caption{Secret as the proportion of Diagnosis = Childbirth.}
    \label{fig:pareto_diagnosis_childbirth}
\end{subfigure}
\begin{subfigure}{0.48\textwidth}
         \centering
    \includegraphics[width=0.81\linewidth]{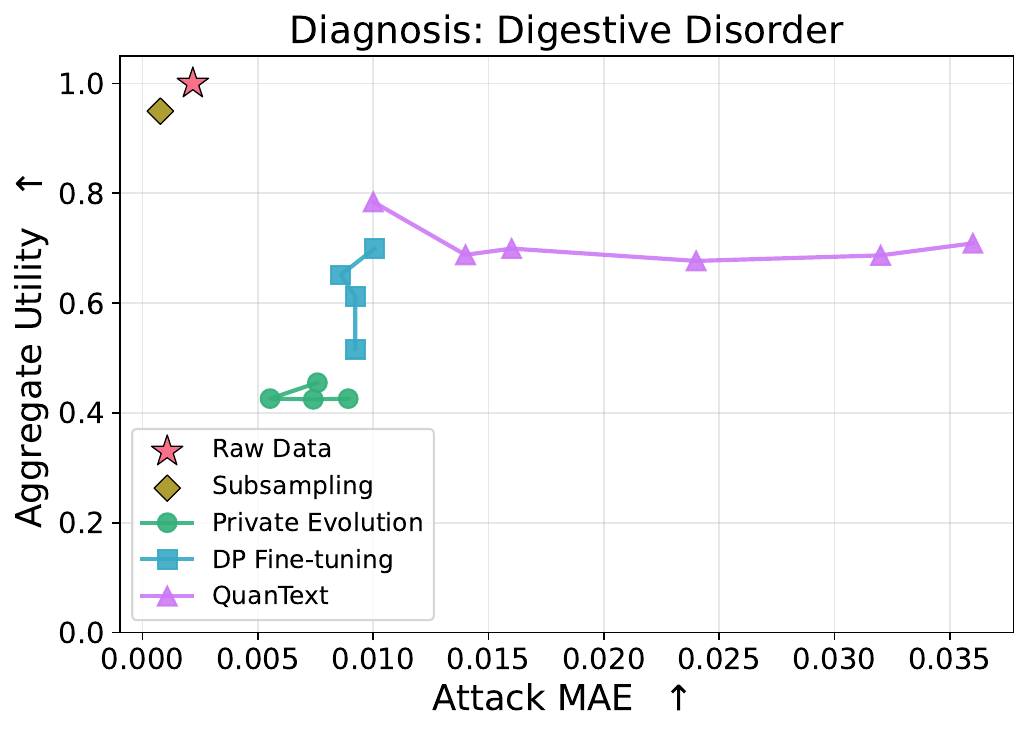}
    \caption{Secret as the proportion of Diagnosis = Digestive Disorder.
    }
    \label{fig:pareto_diagnosis_digestive}
\end{subfigure}
\caption{Attack MAE vs. aggregate utility under different data generation methods with varying privacy guarantees on ChatDoctor. %
}
\label{fig:pareto_chatdoctor}
\vspace{-2mm}
\end{figure*}

\subsection{Evaluation Metrics}

\textbf{Privacy:} 
Although our theoretical arguments suggest that \name satisfies an SML guarantee in idealized settings, it is unclear how to estimate the correlation parameter $l$ from \cref{thm:sml}. Hence, we evaluate an empirical privacy measure that can also be evaluated for the other baseline defenses we consider: 
we calculate the Mean Absolute Error (MAE) of the labeling-based property inference attack used in \cite{hilprecht2019_property_gans,prisampler2024, huang2025can}, described as follows. Higher MAE indicates better performance on protecting global secrets.

\begin{itemize}
\item \textbf{Labeling-based Attack~} For each sample, the attack uses a large language model to decide whether the sample belongs to the sensitive attribute category. The property ratio is then computed as the number of samples in the sensitive category divided by the total number of samples.
\end{itemize}

In our experiments, we use \texttt{Llama3.1-8B-Instruct}, \texttt{Mistral-7B-Instruct}, and \texttt{Qwen3-4B-Instruct} as the backend models for the labeling-based attack.

Although \cite{huang2025can} proposes another property inference attack for text, it targets models trained on private or generated data, where the trained model may encode information about the global secrets. This makes it unsuitable for our setting, in which \name{} is training-free and the attacker infers the secrets directly from the generated dataset.

\textbf{Utility:} Inspired by \cite{wang2026struct,wang2026synae} on evaluating synthetic data, we adopt \knnprecision{}, \knnrecall{}, Fréchet Inception Distance, and Attribute Matching to evaluate the semantic and statistic performance of the generated data. \textbf{\knnprecision{}} and \textbf{\knnrecall{}} capture the semantic quality and coverage of the generated data. 
\knnprecision{} and \knnrecall{} are defined as the proportions of generated and real samples, respectively, whose embedding distance to at least one sample from the opposite dataset is smaller than the distance to the $k$-th nearest neighbor within their own dataset.
\textbf{Fréchet Inception Distance} (\fid{}) measures the embedding closeness of the real and generated data. \textbf{Attribute Match} (\am{}) quantifies the agreement between real and synthetic datasets with respect to predefined statistical and semantic features by measuring distances between their corresponding feature distributions. Specifically, it uses the Wasserstein-2 distance for numerical features and Total Variation (TV) distance for categorical features. In our evaluation, we use sample token length as the statistical feature, together with dataset-specific semantic features.

We summarize overall utility as the average of the four metrics above, each rescaled to $[0,1]$ so that larger values indicate better performance. Although averaging these values may not be ideal because the underlying metrics can have different scales, this practice is sometimes adopted in benchmarks to facilitate visualization and interpretation \cite{wang2023decodingtrust,wang2026synae}.

\begin{figure*}[htbp]
    \centering
\begin{subfigure}{0.48\textwidth}
         \centering
    \includegraphics[width=0.81\linewidth]{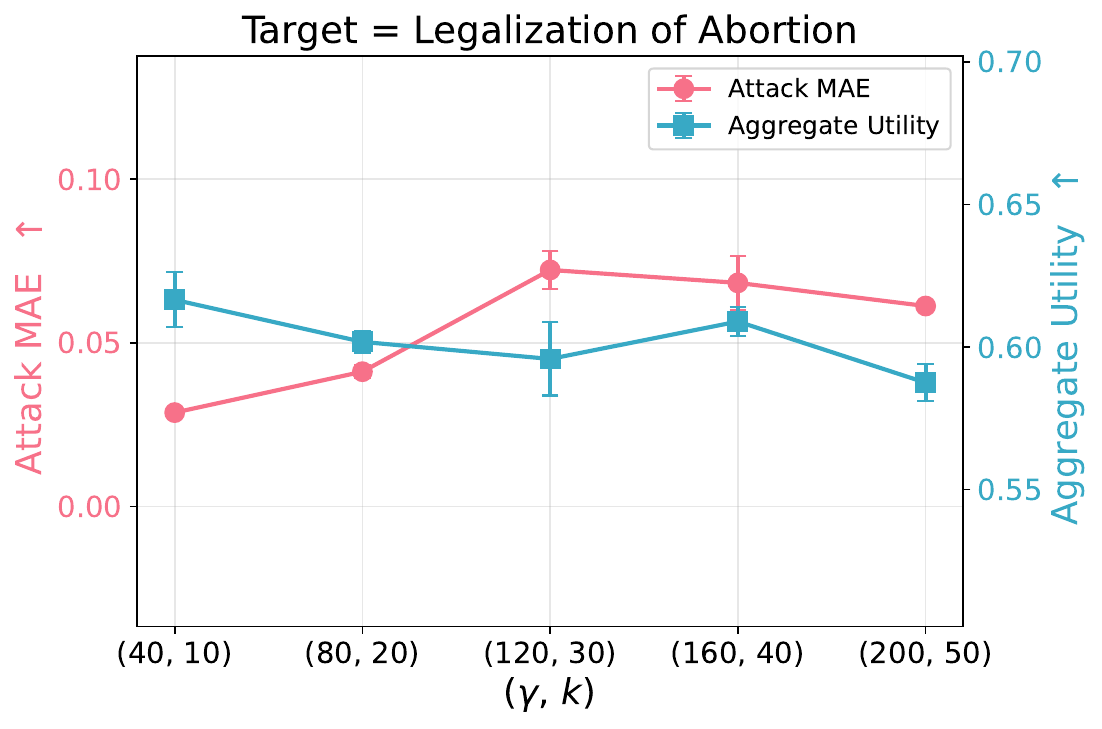}
    \caption{Secret as the proportion of Target = Legalization of Abortion.}
    \label{fig:pareto_target_fix}
\end{subfigure}
\begin{subfigure}{0.48\textwidth}
         \centering
    \includegraphics[width=0.81\linewidth]{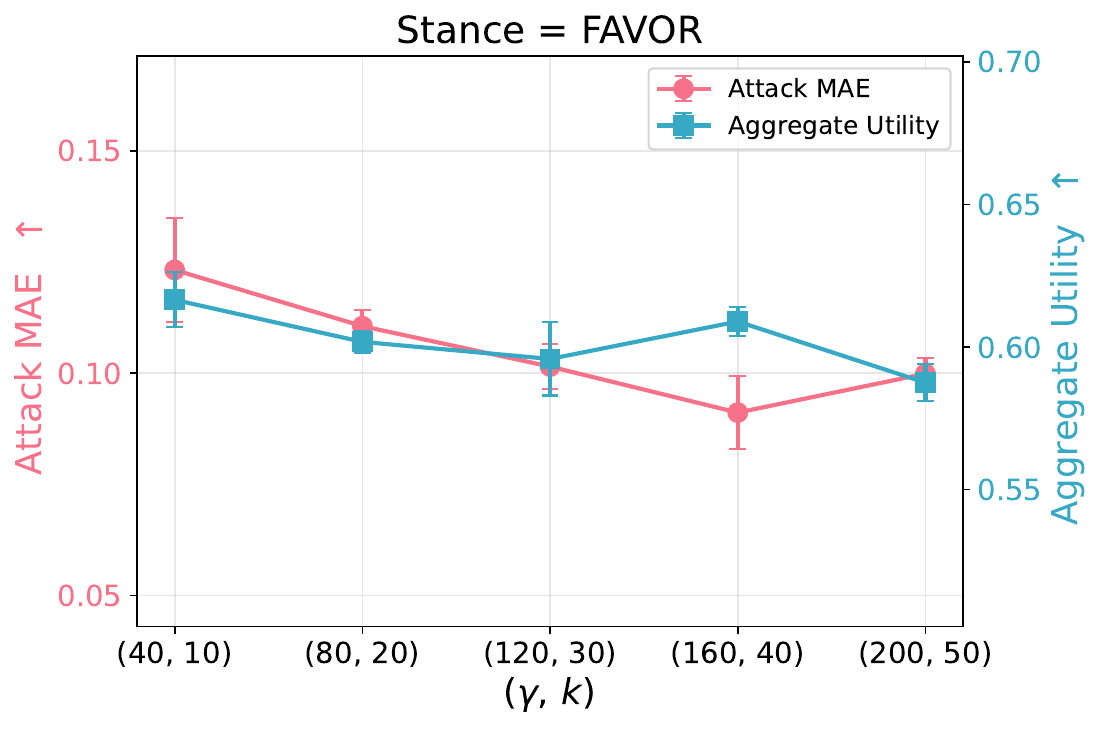}
    \caption{Secret as the proportion of Stance = Favor.}
    \label{fig:pareto_stance_fix}
\end{subfigure}
\begin{subfigure}{0.48\textwidth}
         \centering
    \includegraphics[width=0.81\linewidth]{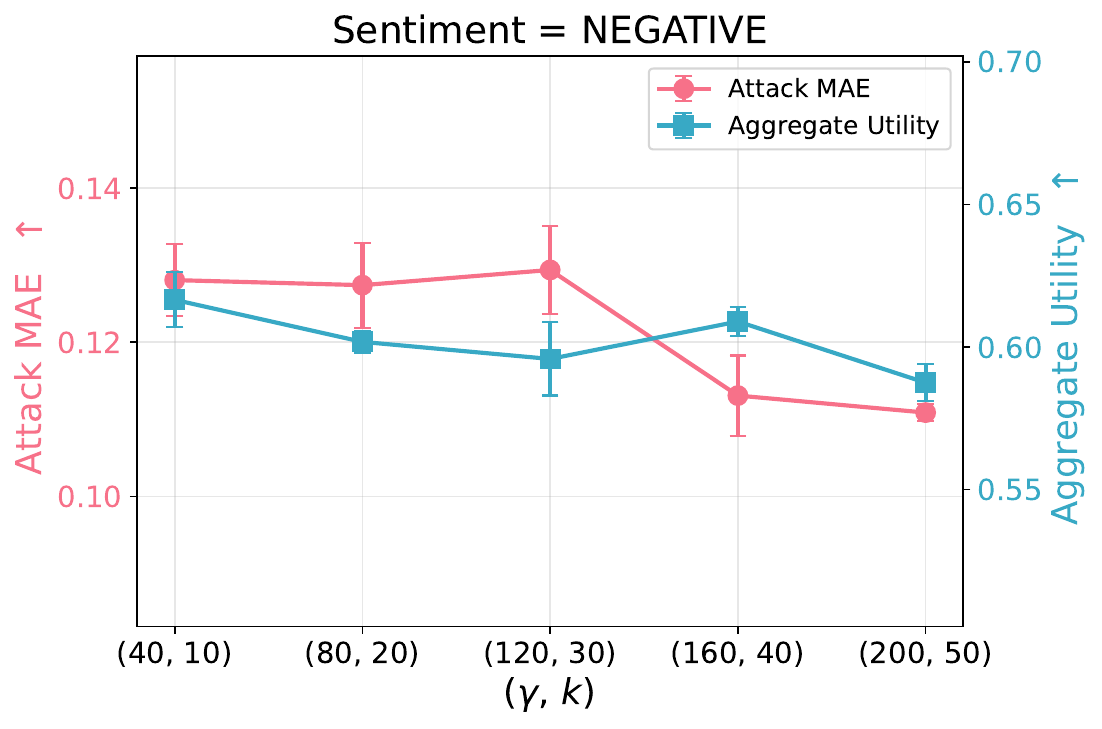}
    \caption{Secret as the proportion of Sentiment = Negative.}
    \label{fig:pareto_sentiment_fix}
\end{subfigure}
\begin{subfigure}{0.48\textwidth}
         \centering
    \includegraphics[width=0.81\linewidth]{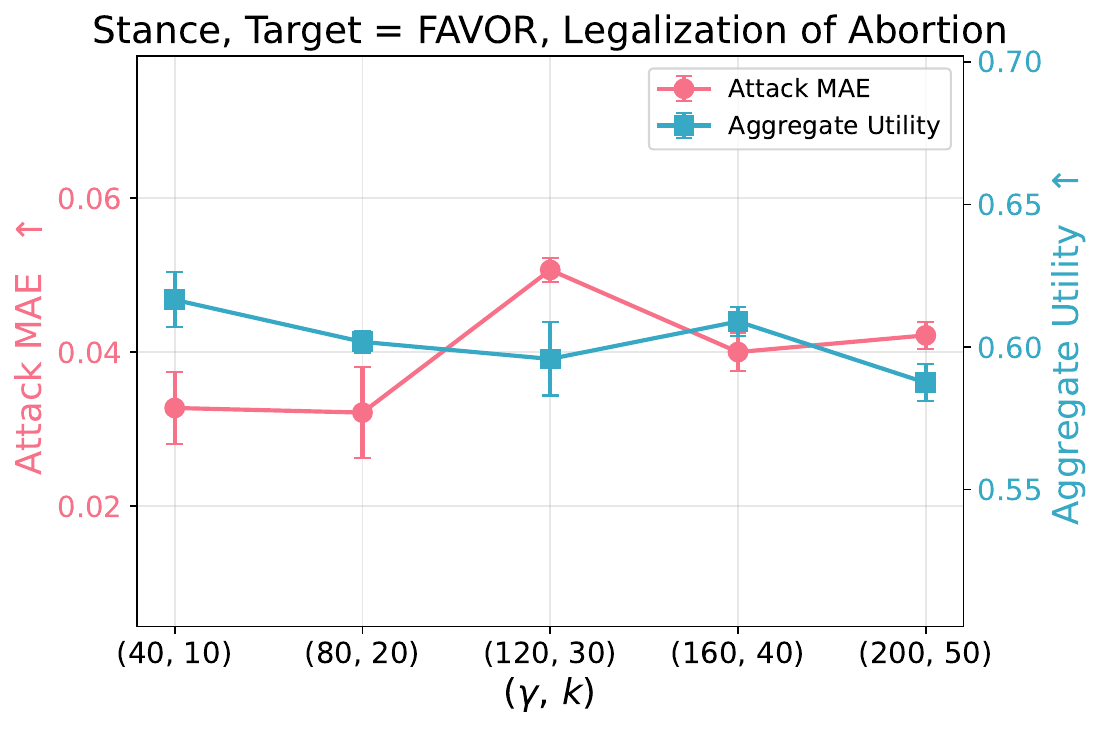}
    \caption{Secret as the proportion of Stance = Favor \& Target = Legalization of Abortion.
    }
    \label{fig:pareto_combined_fix}
\end{subfigure}
\caption{Attack MAE \& aggregate utility of \name{} with ${\gamma}/{k} = 4$ on Tweet. 
}
\label{fig:fix_dilution}
\vspace{-2mm}
\end{figure*}

\subsection{Results}

We present the privacy and utility performance of \name{} and the baselines under varying privacy budgets for selected secrets on the Tweet Stance and ChatDoctor datasets in \cref{fig:pareto_tweet,fig:pareto_chatdoctor}, %
which highlight two main takeaways:
\begin{itemize}
\item \textbf{\name{} achieves a better privacy–utility trade-off than PE and DP-FT.~%
}
As shown in \cref{fig:pareto_tweet,fig:pareto_chatdoctor}, when the candidate dilution satisfies $\frac{\gamma}{k} \leq 8$, %
\name{} consistently achieves both higher attack MAE and a higher aggregate utility score than PE and DP-FT across privacy budgets $\epsilon \in \brc{1,2,3,4}$, %
indicating superior privacy and utility performance. For both PE and DP-FT, the attack MAE remains consistently low and varies only slightly across different DP budgets $\epsilon$, which is consistent with prior observations that data generation methods with differential privacy guarantees are insufficient to protect global-level properties \cite{wang2024statistic,lin2024summary,wang2024guarding}. We also observe that the attack MAE of Raw Data Release and Subsampling remains near zero, indicating that property inference attack succeeds on these naive baselines.

\item \textbf{Smaller candidate dilution $\frac{\gamma}{k}$ improves privacy while preserving utility for \name{}.~} 
As shown in \cref{fig:pareto_target} and consistently observed across the other subplots in \cref{fig:pareto_tweet,fig:pareto_chatdoctor}, decreasing the candidate dilution from $12$ to $2$ substantially improves privacy: the attack MAE increases from approximately $0.01$ to $0.1$. At the same time, this change incurs only a minor utility loss, with the aggregate utility score decreasing from $0.85$ to $0.77$. Intuitively, this is because \name{} primarily adjusts the proportions of selected attributes, while the reused attribute-related snippets help preserve the semantic quality and coverage of the generated samples. A similar phenomenon has been observed in tabular data generation with SML guarantees \cite{wang2024statistic}. These results indicate that smaller candidate dilution provides a more favorable privacy--utility trade-off in practice.
\end{itemize}

\subsubsection{Ablation Studies}
We further study how \name{} performs under different parameter choices and
against different attacker models.

\paragraph{Sensitivity analysis of parameters $\gamma$ and $k$ under fixed candidate dilution}
In \cref{fig:fix_dilution}, we fix the candidate dilution at $\frac{\gamma}{k}=4$ and vary $\gamma$ and $k$ proportionally to examine whether their individual values affect the privacy and utility performance of \name{} on the Tweet Stance dataset. For each parameter setting, we run \name{} three times and report the mean and standard deviation of both privacy and utility. We observe that, across different types of secrets, both the attack MAE and aggregate utility remain largely stable as $\gamma$ and $k$ vary. This demonstrates that \name{} is robust to the specific choices of $\gamma$ and $k$, provided that the candidate dilution is fixed.

\paragraph{Robustness of \name{} to attacks from different backend models}
To examine whether \name{} is robust to attacks performed by different models, we consider three backend models: \texttt{Llama3.1-8B-Instruct}, \texttt{Mistral-7B-Instruct}, and \texttt{Qwen3-4B-Instruct}. %
For each backend model, we perform the labeling-based attack for all secrets on Tweet Stance datasets generated by \name{} with $\frac{\gamma}{k}=4$, PE and DP-FT with $\epsilon=1$, Raw Data Release, and Subsampling. For each secret, we rank these methods according to their attack MAE. We then compute the Spearman correlation between the rankings produced by each pair of backend models for each secret and average the correlations across secrets. The resulting averaged Spearman correlation matrix is shown in \cref{tbl:correlation}. We observe that the Spearman correlation between any pair of models exceeds 0.6, indicating a strong correlation~\cite{evans1996straightforward}. This suggests that the relative protection provided by each method is largely preserved across attack models, implying that the privacy advantage of \name{} is not tied to any particular attack model.

\begin{table}[htbp]
\centering
\caption{Averaged Spearman correlation matrix between backend models. For each backend model and secret, we rank the generation methods by their attack MAE, compute the Spearman correlation between the rankings for each pair of models, and average the correlations across secrets.}
\label{tbl:correlation}
\begin{tabular}{cccc}
\toprule
 & \texttt{Llama} & \texttt{Mistral} & \texttt{Qwen} \\
\midrule
\texttt{Llama}   & 1.0000 & 0.7766 & 0.6878 \\
\texttt{Mistral} & 0.7766 & 1.0000 & 0.7724 \\
\texttt{Qwen}    & 0.6878 & 0.7724 & 1.0000 \\
\bottomrule
\end{tabular}
\end{table}

%% file: sections/conclusion.tex
\section{Conclusion}

In this work, we studied privacy-preserving textual data generation for protecting sensitive global properties, specifically the proportions of samples associated with specified attribute categories. We proposed \name{}, a model-agnostic, training-free mechanism that protects global secrets while limiting leakage through correlated attributes and preserving data utility. \name{} efficiently constructs candidate release distributions, randomly selects one close to the private empirical distribution, and rewrites samples using attribute-related snippets to match the chosen distribution. We analyzed its privacy guarantee using Statistic Maximal Leakage and characterized the privacy degradation when textual style is correlated with the secret. Experiments on real-world datasets show that \name{} achieves a favorable privacy--utility trade-off and outperforms differentially private data generation baselines in both privacy and utility.

%% file: references.bbl

%% file: sections/app_candidate_distribution.tex
\section{Candidate Distribution Construction}
\label{app:candidate_distribution}

We specify the candidate distribution construction algorithm in \cref{alg:candidate_distribution}. Specifically, for each distribution component $\mathbb{P}$, we draw a pool of $L$ distributions from a symmetric Dirichlet distribution with concentration parameter $\alpha$, where the default choice $\alpha=1$ yields the uniform Dirichlet distribution. We then select $\gamma$ distributions from this pool to form the candidate set, with the goal of making the selected candidate distributions as far apart as possible. To this end, we adopt a greedy farthest-point selection rule: at each step, we choose the distribution that maximizes its minimum TV distance to the set of candidates selected so far. The time complexity of \cref{alg:candidate_distribution} is $O(\gamma^2 L)$.

\begin{algorithm}[htbp]
\small
\DontPrintSemicolon
\SetAlgoLined
\SetKwInput{KwIn}{Input}
\SetKwInput{KwOut}{Output}

\KwIn{
Distribution component $\mathbb{P}$; Dirichlet concentration $\alpha$; pool size $L$; number of candidate distributions $\gamma$.
}
\KwOut{Candidate distribution set $\mathcal{Q}$ with size $\gamma$.}

\BlankLine
\tcp{Pool sampling}
Initialize pool $\mathcal{L}\gets\varnothing$.\;
\For{$\ell \gets 1$ \KwTo $L$}{
Draw $\bar{\mathbb{P}}^{(\ell)} \sim \mathrm{Dirichlet}_{\mathrm{sym}}(\alpha)$
over the support of $\mathbb{P}$.\;
$\mathcal{L}\gets \mathcal{L}\cup{\bar{\mathbb{P}}^{(\ell)}}$.\;
}

\BlankLine
\tcp{Greedy selection of $\gamma$ candidates}
Pick an arbitrary $q_1\in\mathcal{L}$ and set $\mathcal{Q}\gets{q_1}$.\;
\For{$t \gets 2$ \KwTo $\gamma$}{
$q^\star \gets
\arg\max_{p\in \mathcal{L}\setminus\mathcal{Q}}
\min_{q\in \mathcal{Q}} d_{\mathrm{TV}}(p,q)$.\;
$\mathcal{Q}\gets \mathcal{Q}\cup{q^\star}$.\;
}

\BlankLine
\Return $\mathcal{Q}$.\;

\caption{Candidate Distribution Construction}
\label{alg:candidate_distribution}
\end{algorithm}
\normalsize

%% file: sections/app_proof_sml.tex
\section{Proof of \cref{thm:sml}}
\label{app:proof_sml}

\begin{proof}
Let $\mathcal{M}'$ be the data release mechanism consisting only of the first four steps of \name{}. By \cite{wang2024statistic}, we have
\begin{align*}
\Pi_{\mathcal{M}', \mathfrak{g}} = \sup_{\mathbb{P}_{\Theta|G}\in\brc{0,1}}\log \sum_{\theta'\in\mathbf{\Theta}'} \sup_{\secretrv\in \secretvalueset} \mathbb{P}_{\Theta'|\Theta}\bra{\theta'|\theta_\secretrv},
\end{align*}
where $\Theta$ and $\Theta'$ denote the distribution parameters of the private and released data respectively, and $\theta_\secretrv$ satisfies $\mathbb{P}_{\Theta|G}\bra{\theta_\secretrv|\secretrv}=1$. Since the number of distribution components is $1+\sum_{i=2}^m \brd{\prop_{\pi\bra{i}}}$, we have $\brd{\mathbf{\Theta}'}=\gamma^{1+\sum_{i=2}^m \brd{\prop_{\pi\bra{i}}}}$ and 
\begin{align*}
\mathbb{P}_{\Theta'|\Theta}\bra{\theta'|\theta}\leq \bra{\frac{1}{k}}^{1+\sum_{i=2}^m \brd{\prop_{\pi\bra{i}}}}, \quad\forall \theta\in\mathbf{\Theta}, \theta'\in\mathbf{\Theta}'.
\end{align*}
Hence, when $L\bra{Y;Z} = 0$, we can get that
\footnotesize
\begin{align*}
\Pi_{\mathcal{M}', \mathfrak{g}} &\leq \sup_{\mathbb{P}_{\Theta|G}\in\brc{0,1}}\hspace{-1mm}\log \sum_{\theta'\in\mathbf{\Theta}'} \hspace{-1mm}\bra{\frac{1}{k}}^{1+\sum_{i=2}^m \brd{\prop_{\pi\bra{i}}}}
\hspace{-1mm}= \log \bra{\frac{\gamma}{k}}^{1+\sum_{i=2}^m \brd{\prop_{\pi\bra{i}}}}.
\end{align*}
\normalsize

Since SML satisfies post-processing and, by \cref{assumption:llm}, the attribute-related snippets reveal no information about the secret, the SML of \name{} satisfies
\begin{align*}
    \Pi_{\mathcal{M}, \mathfrak{g}} \leq \log \bra{\frac{\gamma}{k}}^{1+\sum_{i=2}^m \brd{\prop_{\pi\bra{i}}}}.
\end{align*}

When $L\bra{Y;Z} \leq l$, under \cref{assumption:llm}, we can get that
\footnotesize
\begin{align*}
&\Pi_{\mathcal{M}, \mathfrak{g}}=\sup_{\mathcal{P}, \mathcal{A}}\log \frac{\probof{\hat{G}=G}}{\sup_{\secretrv\in \secretvalueset} \mathbb{P}_G\bra{\secretrv}} \quad
= 
\sup_{\probnotation_{Z|Y}}\sup_{\probnotation_{Y}, \mathcal{A}}\log \frac{\probof{\hat{G}=G}}{\sup_{\secretrv\in \secretvalueset} \mathbb{P}_G\bra{\secretrv}}\\
&= 
\sup_{\probnotation_{Z|Y}}\sup_{\mathbb{P}_{Y|G}\in\brc{0,1}}\log \sum_{\theta'\in\mathbf{\Theta}'} \sup_{\secretrv\in \secretvalueset} \mathbb{P}_{\Theta'|Y}\bra{\theta'|y_\secretrv}\\
&\leq \sup_{\probnotation_{Z|Y}}\sup_{\mathbb{P}_{Y|G}\in\brc{0,1}}\log \sum_{y'\in \mathcal{Y}', z\in \mathcal{Z}}\sup_{\secretrv\in \secretvalueset}\mathbb{P}_{Z|Y}\bra{z|y_\secretrv}\cdot\bra{\frac{1}{k}}^{1+\sum_{i=2}^m \brd{\prop_{\pi\bra{i}}}}\\
&\leq \sup_{\probnotation_{Z|Y}}\sup_{\mathbb{P}_{Y|G}\in\brc{0,1}}\log \sum_{y'\in \mathcal{Y}'}l\cdot\bra{\frac{1}{k}}^{1+\sum_{i=2}^m \brd{\prop_{\pi\bra{i}}}}\\
&= \log \bra{\frac{\gamma}{k}}^{1+\sum_{i=2}^m \brd{\prop_{\pi\bra{i}}}} + \log l,
\end{align*}
\normalsize
where $y_\secretrv$ satisfies $\mathbb{P}_{Y|G}\bra{y_\secretrv|\secretrv}=1$, and $\mathcal{Y}'$ denotes the set of released content parameter vectors.
\end{proof}